\documentclass[letterpaper, 10pt, conference]{ieeeconf}
\usepackage[utf8]{inputenc}
\usepackage{amsmath}
\usepackage{amssymb}
\usepackage{mathtools}
\usepackage{hyperref}

\usepackage{amsthm}
\usepackage{xcolor}
\usepackage{graphicx}
\usepackage{svg}
\usepackage{arydshln}
\usepackage{algorithm}
\usepackage{algpseudocode}

\DeclareMathOperator*{\argmin}{arg\,min}
\DeclareMathOperator*{\argmax}{arg\,max}

\newcommand{\norm}[1]{\left\lVert#1\right\rVert}
\DeclareMathOperator{\R}{\mathbb{R}}

\newcommand{\card}[1]{\left\vert#1\right\vert}
\def\*#1{\mathbf{#1}}
\def\+#1{\mathcal{#1}}
\algdef{SE}[SUBALG]{Indent}{EndIndent}{}{\algorithmicend\ }%
\algtext*{Indent}
\algtext*{EndIndent}

\newtheorem{lemma}{Lemma}

\IEEEoverridecommandlockouts
\title{\LARGE \bf Distributed Multi-Target Tracking for Autonomous Vehicle Fleets}
\author{Ola Shorinwa$^{1}$, Javier Yu$^{2}$, Trevor Halsted$^{1}$, Alex Koufos$^{2}$, and Mac Schwager$^{2}$
\thanks{*This project was funded in part by DARPA YFA award D18AP00064, NSF NRI award 1830402.  Toyota Research Institute (``TRI'') provided funds to assist the authors with their research but this article solely reflects the opinions and conclusions of its authors and not TRI or any other Toyota entity.  The second author was funded on an NSF GRF, and the third on an NDSEG Fellowship.}
\thanks{$^{**}$ The first three authors contributed equally.}%
\thanks{$^{1}$Department of Mechanical Engineering, Stanford University, Stanford, CA 94305, USA, {\tt\small \{shorinwa, halsted\}@stanford.edu}}%
\thanks{$^{2}$Department of Aeronautics and Astronautics, Stanford University, Stanford, CA 94305, USA
        {\tt\small \{javieryu, akoufos, schwager\}@stanford.edu}}%
}

\begin{document}

\maketitle
\thispagestyle{empty}
\pagestyle{empty}

\begin{abstract}
We present a scalable distributed target tracking algorithm based on the alternating direction method of multipliers that is well-suited for a fleet of autonomous cars communicating over a vehicle-to-vehicle network.  Each sensing vehicle communicates with its neighbors to execute iterations of a Kalman filter-like update such that each agent's estimate approximates the centralized \textit{maximum a posteriori} estimate without requiring the communication of measurements. We show that our method outperforms the Consensus Kalman Filter in recovering the centralized estimate given a fixed communication bandwidth.  We also demonstrate the algorithm in a high fidelity urban driving simulator (CARLA), in which 50 autonomous cars connected on a time-varying communication network track the positions and velocities of 50 target vehicles using on-board cameras.
 
\end{abstract}

\section{Introduction}
A key challenge in integrating autonomous vehicles into the transportation infrastructure is ensuring their safe operation in the presence of potential hazards, such as human-operated vehicles and pedestrians. However, tracking the paths of these safety-critical targets using on-board sensors is difficult in urban environments due to the presence of occlusions. Collaborative estimation among networked autonomous vehicles has the potential to alleviate the limitations of each vehicle's individual perception capabilities. Networked fleets of autonomous vehicles operating in urban environments can collectively improve the safety of their planning and decision-making by collaboratively tracking the trajectories of nearby vehicles in real-time.

Constraints on communication and computation impose fundamental challenges on collaborative tracking.  Given limited communication bandwidth, information communicated between vehicles must be succinct and actionable. Communication channels must also be free to form and dissolve responsively given the highly dynamic nature of urban traffic. Relying on centralized computation is neither robust to single points of failure, nor communication-efficient in disseminating information to those vehicles to whom it is relevant.  Rather, a fully-distributed scheme that exploits the computational and communication resources of an autonomous fleet is crucial to reliable tracking.

\begin{figure}[t]
    \centering
    \includegraphics[width=0.65\linewidth]{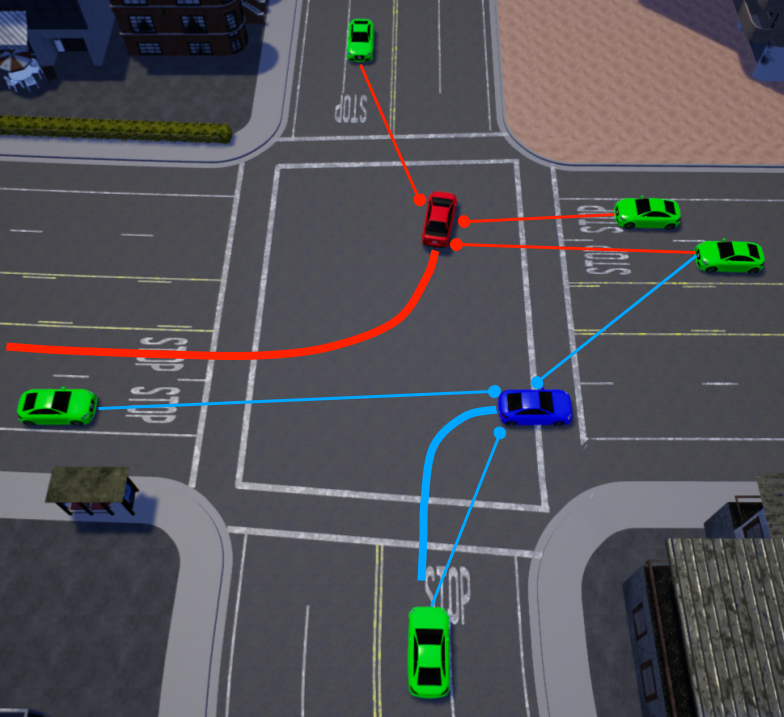}
    \caption{Autonomous vehicles (in green) track the trajectory of target vehicles (in blue and red) with images from on-board cameras at a four-way intersection using our algorithm. }
    \label{fig:motivator}
\end{figure}

In this paper, we consider the problem of distributed target tracking in a fleet of vehicles collaborating over a dynamic communication network, posed as a Maximum A Posteriori (MAP) optimization problem. Our key contribution is a scalable Distributed Rolling Window Tracking (DRWT) algorithm derived from the Alternating Direction Method of Multipliers (ADMM) distributed optimization framework.  The algorithm consists of closed-form algebraic iterations reminiscent of the Kalman filter and Kalman smoother, but guarantees that the network of vehicles converge to the centralized MAP estimate of the targets' trajectories over a designated sliding time window.  We show in extensive simulations that our DRWT algorithm converges to the centralized estimate orders of magnitude faster than a state-of-the art Consensus Kalman Filter for the same bandwidth.  We demonstrate our algorithm in a realistic urban driving scenario in the CARLA simulator, in which 50 autonomous cars track 50 target vehicles in real time using only segmented images from their on-board cameras.

The paper is organized as follows.  We give related work in Sec.~\ref{sec:related_works} and pose the distributed estimation problem in Sec.~\ref{sec:problem}. In Sec.~\ref{sec:centralized estimation}, we formulate the centralized MAP optimization problem, and we derive our DRWT algorithm in Sec.~\ref{sec:distributed_estimation}.  Sec.~\ref{sec:simulations} presents results comparing our DRWT to the Consensus Kalman Filter, and describes large-scale simulations in a CARLA urban driving scenario.

\section{Related Work}
\label{sec:related_works}
Several approaches have previously been applied to solving distributed estimation problems. In distributed filtering methods, consensus techniques enable the asymptotic diffusion of information throughout the communication network, allowing individual computation nodes to approximate the joint estimate in the Consensus Kalman Filter \cite{olfati2005distributed, olfati2007distributed, olfati2009kalman, battistelli2016stability}. Alternatively, using finite consensus techniques can improve communication efficiency \cite{wu2018distributed}. Similar techniques have also been applied to particle filtering \cite{ ong2006decentralised, ahmad2011multi}.  However, the messages communicated in these consensus-based methods contain both information vectors and information matrices, so the communication cost scales superlinearly with the size of the estimate. Our approach recovers the same centralized solution while only communicating the estimate vector.

Sensor fusion techniques accomplish distributed estimation by computing a centralized approximation given individual estimates throughout the network \cite{carli2008distributed}. A key challenge in sensor fusion is keeping track of cross-correlation in conditioning individual estimates on previously-fused centralized estimates \cite{stroupe2001distributed}. Covariance Intersection (CI) addresses this issue by computing a consistent centralized estimate that accounts for any possible cross-correlation between individual estimates  \cite{niehsen2002information, julier2007using, li2013cooperative, noack2017decentralized}. However, in ensuring consistency, CI is often extremely conservative and therefore significantly suboptimal, especially for large networks.

Other estimation techniques approach distributed estimation using optimization. One approach is to aggregate all observations of each target to form a non-linear least squares objective function which recovers the MAP estimate \cite{ahmad2013cooperative}, though such an approach requires all-to-all communication. In \cite{nerurkar2009distributed}, each robot communicates its measurement and state estimate to its neighbors to solve the MAP least-squares problem using the conjugate gradient method. However, this approach still requires each node to communicate its measurements to its neighbors.  Alternatively, some methods have been proposed to divide targets among the trackers using Voronoi partitions \cite{dames2017detecting}, and to track multiple targets using the Probability Hypothesis Density (PHD) filter \cite{dames2017distributed}. 

In this paper, we apply a novel approach to the problem of target tracking.  We pose target tracking as a MAP estimate over a rolling window that bears some similarity to \cite{sibley2006sliding}. We apply ADMM, a technique that allows for distributed optimization of problems with separable objectives, to distribute the resulting MAP optimization problem (see \cite{boyd, mateos} for a detailed survey of ADMM). This approach guarantees convergence to the centralized solution \cite{rockafellar}.

\section{Problem Formulation}
\label{sec:problem}
\subsection{Communication model}
We consider the scenario of $N$ camera-equipped autonomous vehicles (``sensors'') navigating a city that also contains $M$ other vehicles (``targets'').  Each sensor takes measurements of the positions of the targets in its vicinity and can communicate with other nearby sensors.
We model the communication network among the $N$ sensors at time $t$ as a dynamic undirected graph $\+G_t = (\+V, \+E_t)$, with vertices $\+V = \{1, \dots, N\}$ corresponding to sensors and edges $\+E_t$ containing pairs of sensors that can directly share information with each other.  The presence of an edge $(i, j)$ depends on the proximity between sensors $i$ and $j$ at time $t$. The neighbor set ${\+N_{i, t} = \{j \mid (i, j) \in \+E_t\}}$ consists of sensors $j$ that can communicate with sensor $i$ at time $t$.

\subsection{Target assignment}
We assume that the each target in the environment has a unique identifier known to all sensors. This data association task is addressed in \cite{montijano2013distributed}, and can be performed in a completely distributed fashion. 

The set of sensors observing any given target changes due to occlusions coupled with the limited sensing-range of the cameras.  At each time that a sensor observes one or more targets, it generates a set of features for each target (\cite{manzoor2019real}, \cite{hsieh2014symmetrical}, \cite{du2012automatic}) which identify the target. This identifier is communicated to its neighbors. Considering the case of a particular target, we denote the set of sensors that observe it over the time horizon $[t-T, t]$ as $\+W_t$.
The subgraph of sensors that are relevant to the target in the time horizon is $\+G^\prime_t \subseteq \+G_t$, such that $\+V^\prime_t = \+V \cap \+W_t$ and $\+E^\prime_t = \{(i, j) \mid (i,j) \in \+E_t, i, j \in \+V^\prime_t\}$. Sensor $i$ knows that sensor $j$ belongs to $\+V^\prime_t$ since sensor $j$ communicates a descriptor of each observed target.  We assume that the subgraph $\+G^\prime_t$ is connected at all times $t$ (that is, there exists a set of edges that form a path between any $i, j \in \+V^\prime_t$).

\subsection{Distributed estimation}
Given a particular target, each sensor has the task of estimating the target's state $\*x \in \R^n$ which includes its position and velocity over discrete timesteps modeled
as a linear Gaussian system in which
\begin{equation}
\*x_{t+1} = \*A_{t} \*x_{t} + \*w_{t}, \label{eqn:dynamics}
\end{equation}
with linear dynamics $\*A_{t} \in \R^{n \times n}$ and additive noise ${\*w_{t} \sim \+N(\*0, \*Q_{t}) \in \R^n}$. In the following, we represent the trajectory over the time horizon $[t-T, t]$ using the notation ${\*x_{t-T:t} = \begin{bmatrix}
\*x_{t-T}^\top & \cdots & \*x_t^\top
\end{bmatrix}^\top}$. Sensor $i$ makes an observation of the target at time $t$ according to
\begin{equation}\*y_{i, t} = \*C_{i, t} \*x_t + \*v_{i, t}, \label{eqn:ind_measurements}
\end{equation}
with measurement vector ${\*y_{i, t} \in \R^{m_i}}$, measurement matrix $\*C_{i, t} \in \R^{m_i \times n}$, and additive noise ${\*v_{i, t} \sim \+N(\*0, \*R_{i, t}) \in \R^{m_i}}$.  We also refer to the joint set of observations across all sensors in the network as
\begin{equation}
\*y_t = \*C_t \*x_t + \*v_t, \label{eqn:joint_measurements}
\end{equation}
where the joint variables $\*y_t \in \R^m$, $\*C_t \in \R^{m \times n}$, and ${\*v_t \in \R^m}$ are the column-wise concatenations over all $i \in \+V^\prime_t$, of $\*y_{i,t}$, $\*C_{i,t}$ and $\*v_{i,t}$, respectively.

While the joint measurements \eqref{eqn:joint_measurements} are not available to any single sensing agent, each agent uses its individual measurements \eqref{eqn:ind_measurements} as well as communication with its neighbors to estimate the target's state.  We compare the sensor's estimated mean and covariance with the mean and covariance computed with full knowledge of all measurements. In the \textit{distributed estimation problem}, each sensor seeks to approximate the centralized (best-possible) estimate using only individual measurements and local communication.

\section{Centralized Estimation}
\label{sec:centralized estimation}
The centralized estimate, which is conditioned on all measurements and priors in the network, gives the best estimate of a target's state and therefore represents the best possible performance. The MAP batch estimate maximizes the probability of the estimated target trajectory $\*x_{0:t}$ conditioned on the full set of measurements $\*y_{0:t}$ and a prior of mean $\bar{\*x}_0$ and covariance $\bar{\*P}_0$:
\begin{align}
\hat{\*x}_{0:t} &= \argmax_{\*x_{0:t}} p(\*x_{0:t} \mid \bar{\*x}_0, \*y_{1:t}) \\
&= \argmax_{\*x_{0:t}} p(\*x_0 \mid \bar{\*x}_0)\prod_{\tau = 0}^{t-1} p(\*x_{\tau+1} \mid \*x_{\tau}) \prod_{\tau = 1}^t p(\*y_\tau \mid \*x_\tau). \label{eqn:mapprob}
\end{align}
Given Gaussian conditional probabilities, the posterior in \eqref{eqn:mapprob} is the Gaussian distribution  $\+N(\hat{\*x}_{0:t}, \hat{\*P}_{0:t})$.

In the case of linear Gaussian systems, we can solve \eqref{eqn:mapprob} as a linear system of equations.  However, recursively estimating the trajectory reduces the size of the system of equations, improving computational efficiency. Instead of maximizing ${p(\*x_{t} \mid \bar{\*x}_0, \*y_{0:t})}$, the Kalman filter infers $\hat{\*x}_t$ from the result of the previous timestep's estimate, the prior distribution $(\bar{\*x}_{t-1}, \bar{\*P}_{t-1})$:
\begin{equation}
\hat{\*x}_{t} = \argmax_{\*x_{t}} p(\*x_{t} \mid \bar{\*x}_{t-1}, \*y_{t}).
\end{equation}
However, the Kalman filter only exactly replicates the result of the batch estimate for the final timestep $t$.  For some intermediate $\tau < t$, $\hat{\*x}_\tau$ is conditioned on the full measurement set $\*y_{0:t}$ in the batch approach, but only on $\*y_{0:\tau}$ in the filtering approach.  Employing the Rauch-Tung-Striebel smoother exactly recovers the batch solution by computing $p(\hat{\*x}_{\tau} \mid \hat{\*x}_{\tau+1})$ for ${\tau = t-1, \dots, 0}$ (a backward pass of the trajectory performed after the Kalman filter's forward pass).

For our application of persistently tracking targets, a MAP rolling window approach is appropriate as it incorporates smoothing effects into a single Kalman filter-like update.  The rolling window refers to a time horizon ${[t-T, t]}$ over which we compute the MAP estimate.  Given the prior $(\bar{\*x}_{t-T:t-1}, \bar{\*P}_{t-T:t-1})$, we compute the window's posterior distribution by factoring the original MAP solution as
\begin{align}
\begin{split}\hat{\*x}_{t-T:t} &= \argmax_{\*x_{t-T:t}} \big\{p(\*x_{t-T:t-1} \mid \bar{\*x}_{t-T:t-1}) \\ &  \hspace{4em}  p(\*x_t \mid \*x_{t-1}) p(\*y_t \mid \*x_t)\big\}. \end{split}\label{eqn:rollingwindowdist}
\end{align}

The estimate $\hat{\*x}_{t-T:t}$ is conditioned on $\bar{\*x}_0$, $\*y_{0:t}$ and is equivalent to performing a filtering pass for the times $0, \dots, t$ and a smoothing pass from time $t$ to time $t-T$. We then increment the rolling window forward to ${[t-T+1,  t+1]}$, retaining the estimate $(\hat{\*x}_{t-T+1:t}, \hat{\*P}_{t-T+1:t})$ as the prior for that window.  Therefore, the rolling window approach preserves much of the smoothing effect of the batch estimate while maintaining a constant problem size at each time step.

Applying \eqref{eqn:dynamics} and \eqref{eqn:joint_measurements} to \eqref{eqn:rollingwindowdist} yields
\begin{align}
J\left(\hat{\*x}_{t-T:t}\right) &= \norm{\hat{\*x}_{t} - \*A_{t-1} \hat{\*x}_{t-1}}_{\*Q_{t-1}^{-1}}^2 + \norm{\*y_t - \*C_t \hat{\*x}_t}_{\*R_t^{-1}}^2 \nonumber \\ & \qquad + \norm{\hat{\*x}_{t-T:t-1} - \bar{\*x}_{t-T:t-1}}_{\bar{\*P}_{t-T:t-1}^{-1}}^2. \label{eqn:cent_cost_func}
\end{align} for which the minimizing $\hat{\*x}_{t-T:t}$ is the solution to \eqref{eqn:rollingwindowdist}.

We can express the MAP rolling window estimate as
\begin{align}
\label{eqn:MAP_estimator_soln_mean}
\hat{\*x}_{t-T:t} &= \left(\*H_t^\top \*W_t^{-1} \*H_t\right)^{-1}\*H_t^\top \*W_t^{-1} \*z_t\\
\hat{\*P}_{t-T:t} &= \left(\*H_t^\top \*W_t^{-1} \*H_t\right)^{-1}, \label{eqn:MAP_estimator_soln_cov}
\end{align}
given the block matrices
\[
\*H_t = 
\left[\begin{array}{c c c c}
 \*0 & \cdots & -\*A_{t-1} & \*I \\ \hline
 \*0 & \cdots &  \*0 & \*C_t \\ \hline
\*I & \cdots  & \*0 & \*0 \\
 \vdots & \ddots & \vdots & \vdots \\
\*0& \cdots& \*I & \*0 
\end{array}\right]
= \left[\begin{array}{c}
\*F_t   \\ \*G_t\\ \*\Pi_t
\end{array}\right], \]
\[\*z_t = \left[\begin{array}{c  c c}
\*0^\top  & \*y_t^\top & \bar{\*x}_{t-T:t-1}^\top
\end{array}\right]^\top, \]
\[\*W_t = 
\text{blkdiag}\left(
\*Q_{t-1}, \*R_t, \hat{\*P}_{t-T:t-1}\right).\]

We implement this procedure recursively by retaining the lower-right block of the covariance matrix $\hat{\*P}_{t-T:t}$ as the prior covariance $\bar{\*P}_{t-T+1:t}$ for the next timestep's estimate. The estimate over all but timestep $t-T$ becomes the prior mean $\bar{\*x}_{t-T+1:t}$. Therefore, we have a tractable centralized target tracking method that serves as a benchmark for our distributed target tracking algorithm.

\section{Distributed Estimation}
\label{sec:distributed_estimation}
One typical approach for the distributed implementation of the MAP estimate is to use consensus techniques to diffuse information across the network, enabling each agent to minimize \eqref{eqn:cent_cost_func}.  This is true of Consensus Kalman Filter (CKF) approaches, in which each agent maintains local measurement information \eqref{eqn:ind_measurements} rather than the joint measurements \cite{olfati2007distributed,  olfati2009kalman, battistelli2016stability, olfati2005distributed, wu2018distributed}. The CKF uses asymptotic consensus with Metropolis weights to sum $\*G_{i, t}^\top \*R_{i, t}^{-1} \*G_{i, t}$ and $\*G_{i, t}^\top \*R_{i, t}^{-1} \*y_{i, t}$ over all $i \in \+V_t^\prime$, where ${\*G_{i,t} = \begin{bmatrix} \*0 & \dots & \*0 & \*C_{i, t} \end{bmatrix}}$. The fused observations are then fused with local copies of the dynamics terms and prior terms of the cost function. The consensus rounds diffuse the joint measurement information to each sensor, enabling local computation of \eqref{eqn:MAP_estimator_soln_mean} and \eqref{eqn:MAP_estimator_soln_cov}.

The CKF requires communication of local information matrices and information vectors during consensus, a communication-intensive process that is a drawback of the method.  Furthermore, performing an approximation of the centralized estimate at each node is redundant, failing to take advantage of the distributed nature of the computational resources in the network. In contrast to the CKF, we propose a Distributed Rolling Window Tracking (DRWT) algorithm that uses an ADMM-based approach to enable each sensor to replicate the centralized estimate without reconstructing the centralized cost function.  First, we pose the centralized cost function \eqref{eqn:cent_cost_func} as a separable problem with linear constraints:
\begin{equation}
\begin{split}
\underset{\hat{\*x}_{i, t-T: t} \forall i \in \+V^\prime_t}{\text{minimize}} & \sum_{i \in \+V^\prime_t} \bigg\{\frac{1}{\vert \+V_t^\prime\vert}\norm{\hat{\*x}_{i, t} - \*A_{t-1} \hat{\*x}_{i, t-1}}_{\*Q_{t-1}^{-1}}^2 \bigg. \\ & \quad + \norm{\*y_{i, t} - \*C_{i,t} \hat{\*x}_{i, t}}_{\*R_{i, t}^{-1}}^2  \\ & \bigg. \quad + \norm{\hat{\*x}_{i, t-T:t-1} - \bar{\*x}_{i, t-T:t-1}}_{\bar{\*P}_{i, t-T:t-1}^{-1}}^2\bigg\}
\\ \text{subject to}\quad & \hat{\*x}_{i, t-T:t}  = {\*r_{ij}}  \qquad \forall j \in \+N_{i, t} \\
 & \hat{\*x}_{j, t-T:t} = {\*r_{ij}} \qquad \forall j \in \+N_{i, t},  \label{eqn:dist_prob}
 \end{split}
\end{equation}
for which $\sum_{i \in \+V_t^\prime} \bar{\*P}^{-1}_{i, t-T:t-1} = \bar{\*P}^{-1}_{ t-T:t-1}$. In the following, we express the cost function in \eqref{eqn:dist_prob} as $\sum_{i \in \+V^\prime_t} J_i(\*x_{i, t-T:t})$ and omit the subscript ${t-T:t}$ from the primal variable $\hat{\*x}_i$. The slack variable $\*r_{ij} \in \R^{n(T+1)}$ encodes agreement constraints between neighbors $i$ and $j$.  The ADMM approach to solving problems of this form uses the augmented Lagrangian, which adds to the cost function a quadratic penalty for constraint violations, ${\sum_{i \in \+V_t^\prime}\sum_{j \in\+N_{i,t}}(\rho/2)(\norm{\*x_{i} - {\*r_{ij}}}^2 + \norm{\*x_{j} - {\*r_{ij}}}^2)}$. The augmented problem is equivalent to the original problem as the added penalty is zero for the feasible set of estimates. We find the saddle point of the augmented Lagrangian
\begin{align}
L_\rho &= \sum_{i \in \+V^\prime_t}\left.J_i\left(\*x_{i}\right)\right. 
 + \sum_{j \in \+N_{i, t}} \left(\*\lambda_{ij}^\top\left(\*x_{i} - \*r_{ij}\right) + \*\mu_{ij}^\top\left(\*x_{j} - \*r_{ij}\right) \right) \nonumber \\ &\qquad + \frac{\rho}{2}\sum_{j \in \+N_{i, t}}\left(\norm{\*x_{i} - {\*r_{ij}}}^2 + \norm{\*x_{j} - {\*r_{ij}}}^2\right) \label{eqn:L_rho}
\end{align}
by alternating between minimizing $L_\rho$ with respect to the primal variables $\*x$ and $\*r$ and performing a gradient ascent step on the dual variables $\lambda_{ij}$ and $\mu_{ij}$.  Each $i\in \+V^\prime_t$ can update its respective $\*x_i$, $\*r_{ij}$, $\lambda_{ij}$, and $\mu_{ij}$ for all $j \in \+N_{i, t}$ in parallel since minimizing $L_{\rho}$ with respect to these variables does not depend on the values of its neighbors' variables. Furthermore, as shown in \cite{mateos, chang2014multi}, substituting $\*p_i = \sum_{j \in \+N_{i, t}} \lambda_{ij} + \mu_{ij}$ and assuming the initialization $\*p_{i}^{(0)} = \*0$ yields the minimization of $\*r_{ij}$ as ${\frac{1}{2}( \*x_{i} + \*x_{j}) }$. Initializing ${{\*p}_i^{(0)} = \*0}$ and ${\hat{\*x_i}^{(0)} = \argmin_{\*x_i} J_i(\*x_i)}$, the following iterations alternate between a gradient ascent step on $\*p_i$ and a minimization step on $\*x_i$, converging to the centralized estimate when run in parallel across all $i \in \+V^\prime_t$:
\begin{align}
{\*p}^{(k+1)}_i &= {\*p}_i^{(k)} + \rho\sum_{j \in \+N_{i, t}} \left(\hat{\*x}^{(k)}_{i} - \hat{\*x}^{(k)}_{j}\right) \label{eqn:ADMM_update_p}\\
\begin{split}
\hat{\*x}^{(k+1)}_{i} &= \argmin_{\*x_{i}} \bigg\{J_i\left(\*x_{i}\right) + \*x_{i}^\top {\*p_i^{(k+1)}} \bigg.\\
&\bigg.\qquad + \rho\sum_{j \in \+N_{i, t}}\norm{\*x_{i} - {\frac{1}{2}\left( \hat{\*x}^{(k)}_{i} + \hat{\*x}^{(k)}_{j}\right) }}^2 \bigg\}\label{eqn:ADMM_update_x}
\end{split}
\end{align}
Furthermore, due to our assumption of a linear Gaussian system \eqref{eqn:ADMM_update_x} can be expressed in closed form as \begin{equation}
\begin{split}
&\left(\*H_{i}^\top \*W_{i}^{-1} \*H_{i} + 2\rho \vert \+N_{i} \vert \*I\right)\hat{\*x}^{(k+1)}_{i} = \\
&\qquad \*H_{i}^\top \*W_{i}^{-1} \*z_{i} - \*p_i^{(k+1)} + \rho \sum_{j \in \+N_{i}} \left(\hat{\*x}_{i}^{(k)} + \hat{\*x}_{j}^{(k)}\right).
\end{split} \label{eqn:ADMMupdate}
\end{equation}
using the local versions of the block matrices in \eqref{eqn:MAP_estimator_soln_mean} (replacing $\*C_t$, $\*Q_{t-1}$, $\*R_t$, $\bar{\*P}_{t-T:t-1}$, $\*y_t$, and $\bar{\*x}_{t-T:t-1}$ with $\*C_{i, t}$, $\*Q_t/{\vert \+V^\prime_t\vert}$, $\*R_{i, t}$, $\bar{\*P}_{i, t-T:t-1}$, $\*y_{i, t}$, and $\bar{\*x}_{i, t-T:t-1}$, respectively).
We note  that the matrix inverse in \eqref{eqn:ADMMupdate} only needs to be computed once rather than at every primal update iteration.

\begin{lemma} \label{lem:SP} Given a connected $\+G_t^\prime$ and priors $\bar{\*x}_{i, t-T:t-1}$ and $\bar{\*P}_{i, t-T:t-1}$ such that ${\bar{\*x}_{i, t-T:t-1} = \bar{\*x}_{t-T:t-1}} \:{\forall i \in \+V^\prime_t}$ and ${\sum_{i \in \+V_t^\prime} \bar{\*P}^{-1}_{i, t-T:t-1} = \bar{\*P}^{-1}_{t-T:t-1}}$, there is a saddle point of \eqref{eqn:L_rho} at \begin{align}
\hat{\*x}^{(k)}_{i} &= \hat{\*x}_{t-T:t} \label{eqn:primal_opt}
\\\*p^{(k)}_{i} &= \*H_i^\top \*W_i^{-1} \*z_i - \*H_i^\top \*W_i^{-1} \*H_i\hat{\*x}_{t-T:t} ,\label{eqn:dual_opt} \end{align}
	where $\hat{\*x}_{t-T:t}$ is the centralized MAP rolling window estimate given priors $\bar{\*x}_{i, t-T:t-1} $, $ \bar{\*P}_{i, t-T:t-1}$.
\end{lemma}

\begin{proof}
	The Hessian of $L_\rho$ with respect to the primal variables $\hat{\*x}_i \: \forall i \in \+V^\prime_t$ is positive definite. Observing that
	\begin{equation*}
	\left.\frac{\partial L_\rho}{\partial \hat{\*x}_i} \right|_{\hat{\*x}_i = \hat{\*x}_{t-T:t} } = \frac{\partial}{\partial \hat{\*x}_i} J(\hat{\*x}_{t-T:t}) \\
	= \frac{\partial}{\partial \hat{\*x}_i} J(\hat{\*x}_{t-T:t}) = 0
	\end{equation*}
for each $i \in \+V^\prime_t$, $L_\rho$ is minimized with respect to the primal variables at $\hat{\*x}_i  = \hat{\*x}_{t-T:t} \: {\forall i \in \+V^\prime_t}$. Substituting the primal solution into the dual update, we see that $\frac{\partial L_\rho}{\partial \*p_i} = 0$. Substituting \eqref{eqn:primal_opt} into \eqref{eqn:ADMMupdate} yields \eqref{eqn:dual_opt}.
\end{proof}
In other words, the network can minimize \eqref{eqn:cent_cost_func} in a fully distributed manner using only independent measurements and local communication.  By decomposing the centralized problem according to \eqref{eqn:dist_prob}, each estimate $\hat{\*x}_{i}$ converges to the solution of \eqref{eqn:MAP_estimator_soln_mean}. A key assumption, however, is the decomposability of prior information, \textit{i.e.}, ${\sum_{i \in \+V_t^\prime} \bar{\*P}^{-1}_{i, t-T:t-1} = \bar{\*P}^{-1}_{ t-T:t-1}}$.
Given that the distributed prior inverse covariances sum to the centralized prior inverse covariances, then the distributed posterior inverse covariances (where $\hat{\*P}_{t-T:t}$ is the Hessian of the local cost function $J_i$) also sum to the centralized posterior inverse covariances.  However, this assumption weakens in implementing DRWT recursively.  In performing the marginalization step in which $\bar{\*P}_{t-T+1:t}$ is the ${t-T+1:t}$ block of $\hat{\*P}_{t-T:t}$, the distributed implementation is not exactly equivalent to the centralized. It is always true that ${(\sum_{i \in \+V_t^\prime} \bar{\*P}^{-1}_{i, t-T:t-1})^{-1} \ge \bar{\*P}_{t-T:t-1}}$. Consequently, the distributed marginalization is conservative with respect to the centralized solution. The conservativeness of the estimated covariance is a feature of other distributed algoithms as well---as Figure \ref{fig:mc-abs-error} shows, the CKF has an even more conservative covariance estimate.  Therefore, while DRWT remains an unbiased estimator, it does not exactly replicate the centralized covariance in its recursive implementation, as the prior mean is under-weighted.  Lemma \ref{lem:SP} holds, with the modification that the saddle point is the solution to a centralized optimization problem with a potentially overestimated prior covariance. As we show in Sec.~\ref{sec:simulations}, this effect is minimal in practice.

Finally, we propose a ``hand-off'' protocol by which sensor $i$ removes itself from estimating a target after not directly observing it in the $T$ most recent timesteps.  If there exists ${j \in \+N_{i, t} \cap \+V^\prime_{t+1}}$ (\textit{i.e.}, neighbor $j$ is continuing to estimate the target), then $i$ transfers the Hessian of its local cost function to a single neighbor $j$ at the end of the ADMM iterations. Sensor $j$ fuses the new information matrix with its own, thereby preserving the same joint information across the entire network. Algorithm \ref{alg:DRWT}
summarizes DRWT, including the hand-off protocol.

\begin{algorithm}[t]
	\caption{Distributed Rolling Window Tracking}\label{alg:DRWT}
	\begin{algorithmic}[1]
		\Function{DRWT}{$\bar{\*x}_{i, t-T:t-1}, \bar{\*P}_{i, t-T:t-1}, \*y_{i, t} \quad \forall i \in \+V^\prime_t$}
		\For{$i \in \+V^\prime_t$}
		\State $\hat{\*x}_{i, t-T:t}^{(0)} \gets \argmin_{\*x_{i, t-T:t}} J_i(\*x_{i, t-T: t})$
		\State $\*p_i^{(0)} \gets \*0$
		\State $\hat{\*P}_{i, t-T:t} \gets \left(\*H_{i, t}^\top \*W_{i, t}^{-1} \*H_{i, t}\right)^{-1}$
		\EndFor
		\While{stopping criterion is unmet}
		\For{$i \in \+V^\prime_t$}
		\State $\*p_i^{(k+1)} \gets$ Equation \eqref{eqn:ADMM_update_p} \Comment{dual update}
		\State $\hat{\*x}_{i, t-T:t}^{(k+1)} \gets$ Equation \eqref{eqn:ADMM_update_x} \Comment{primal update}
		\EndFor
		\State $k \gets k+1$
		\EndWhile
		\For{$i \in \+V_t^\prime \notin \+V_{t+1}^\prime, j \in \+N_{i, t}\cap \+V_{t+1}^\prime$}
		\State $\hat{\*P}_{j, t-T:t} \gets \left(\hat{\*P}_{i, t-T:t}^{-1} + \hat{\*P}_{j, t-T:t}^{-1} \right)^{-1}$ \Comment{hand-off}
		\EndFor
		\State \Return $\hat{\*x}_{i, t-T:t}, \hat{\*P}_{i, t-T:t} \quad \forall i \in \+V^{\prime}_t$
		
		\EndFunction
	\end{algorithmic}
\end{algorithm}

After each communication round per timestep, sensor $i$ updates its estimate of the target's trajectory (\ref{eqn:ADMM_update_x}) by inverting the Hessian of its local objective function which requires ${O(n^{3}(T+1)^{3})}$ floating point operations (flops), posing a bottleneck for long window lengths. Here, we provide an efficient algorithm for performing this update in ${O(n(T+1))}$ flops rather than cubic complexity, without any matrix inversion. We factor the Hessian using Cholesky decomposition to obtain a lower triangular matrix $\*L_{\tau}$ for each ${\tau = t-T,\cdots,t}$ and compute $\pmb\sigma_{\tau}$ to update $\hat{\*x}_{i, t-1:t}$ using forward and backward iterations, reminiscent of the Kalman smoothing procedure. The Cholesky decomposition of the Hessian takes ${O(n(T+1))}$ flops, along with the forward and backward iterations. We present the algorithm in Algorithm \ref{alg:DRWT_smoother}.


\begin{algorithm}[t]
	\caption{DRWT Primal Update}\label{alg:DRWT_smoother}
	\begin{algorithmic}[1]
		\Function{PrimalUpdate}{$\*p_{i}^{(k)}, \hat{\*x}_{i}, \hat{\*x}_{j} \: \forall j \in \+{N}_{i,t}$}
		    		
    		\State $\pmb\alpha_{\tau} \coloneqq \frac{\rho}{2} \sum_{j \in \+{N}_{i}} \left(\hat{\*x}_{i,\tau}^{(k)} + \hat{\*x}_{j,\tau}^{(k)} \right) - \frac{1}{2}\*p_{i,\tau}^{(k)}$
    		\State $\pmb\gamma_{\tau} \coloneqq \frac{1}{\card{\+{V}_{t}^{\prime}}}\*Q_{\tau-1}^{-1}+ \rho \card{\+{N}_{i}}$
    		
		\State \textbf{initialization}
		\Indent
    		\State $ \*\Phi_{T-t} \gets \hat{\*P}_{i,T-t}^{-1}\bar{\*x}_{i,T-t} + \rho\card{\+{N}_{i}}$
    		\State $ \pmb \beta_{i,T-t} \gets \hat{\*P}_{i,T-t}^{-1}\bar{\*x}_{i,T-t} +  \pmb\alpha_{T-t}$
		\EndIndent
		\State \textbf{forward pass}
		\Indent
		\For{$\tau = t-T+1,\cdots,t$}
    		\State $\*L_{\tau-1}\*L_{\tau-1}^\top \gets \*\Phi_{\tau-1} + \frac{1}{\card{\+{V}_{t}^{\prime}}}\*{A}^{\top}_{\tau-1} \*Q_{\tau-1}^{-1} \*A_{\tau-1}$
    		\State $\*L_{\tau,\tau-1}\*L_{\tau-1}^\top \gets -\frac{1}{\card{\+{V}_{t}^{\prime}}} \*Q_{\tau-1}^{-1} \*A_{\tau-1}$
    		\State $\*L_{\tau-1}\pmb\sigma_{\tau-1} \gets \pmb\beta_{i,\tau-1}$
    		\State $\*\Phi_{\tau} \gets -\*L_{\tau,\tau-1}\*L_{\tau,\tau-1}^\top + \hat{\*P}_{i,\tau}^{-1}  + \pmb\gamma_{\tau}$
    		\State $\pmb\beta_{i,\tau} \gets -\*L_{\tau,\tau-1}\pmb\sigma_{\tau-1} + \hat{\*P}_{i,\tau}^{-1}\bar{\*x}_{i,\tau} + \pmb\alpha_{\tau}$	
		
		\EndFor
		\State $\*L_{t}\*L_{t}^\top \gets -\*L_{t,t-1}\*L_{t,t-1}^\top + \*C_{t}^{\*T}\*R_{i,t}^{-1}\*C_{t} + \pmb\gamma_{t}$
		\State $\*L_{t}\pmb\sigma_{t} \gets  -\*L_{t,t-1}\pmb\sigma_{t-1} + \*C_{t}^{\*T}\*R_{i,t}^{-1}\*y_{i,t} + \pmb\alpha_{t}$
		\State $ \hat{\*x}_{i,t}^{(k+1)} \gets -\*L_{t}^{-\top} \pmb\sigma_{t} $
		\EndIndent
		\State \textbf{backward pass}
		\Indent
		\For{$\tau = t,\cdots,t-T+1$}
            \State $ \hat{\*x}_{i,\tau-1}^{(k+1)} \gets -\*L_{\tau-1}^{-\top}\left(\*L_{\tau,\tau-1}\hat{\*x}_{i,\tau}^{(k+1)} + \pmb\sigma_{\tau-1}\right) $
		\EndFor
		\EndIndent
        \State \Return $\hat{\*x}_{i, t-T:t}^{(k+1)}$
		\EndFunction
	\end{algorithmic}
\end{algorithm}

\section{Simulation Results}
\label{sec:simulations}

\subsection{Performance Comparison}
We compare the performance of the DRWT method  in Algorithm \ref{alg:DRWT} to the CKF in a distributed estimation problem involving a static network with $\vert\+V\vert = 100$ and $\vert \+E\vert = 400$. All sensors acquire noisy measurements of the target at each time step, and perform DRWT with $T = 1$. During each estimation phase, the same bandwidth limitations are imposed on the CKF and DRWT. We benchmark both distributed methods against the centralized MAP estimate.

Results from 2000 Monte Carlo simulations of this scenario show that DRWT method outperforms the CKF. DRWT is significantly more communication-efficient, as sensors communicate only their target estimates. From Figure \ref{fig:comp-convergence},
DRWT yields better convergence to the centralized estimate compared to the CKF method as a function of the total number of communication bits per node.
As Figure \ref{fig:mc-abs-error} shows, the improved convergence of the DRWT contributes to improved estimation performance over entire trajectories.  The estimated trajectories and covariances of the DRWT method closely match the centralized estimates.  The CKF does not track the centralized estimate as closely and is also more significantly overconservative in its estimate.

\begin{figure}[thpb]
    \centering
    \includegraphics[width=0.9\columnwidth]{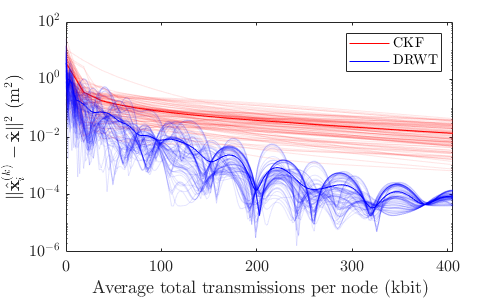}
    \caption{Convergence of distributed estimation methods to the centralized estimate as a function of bits of communication passed on a 100 node, 400 edge network for a single timestep's estimate.}
    \label{fig:comp-convergence}
\end{figure}

\begin{figure}[thpb]
    \centering
    \includegraphics[width=0.9\columnwidth, trim=0 0 0 0, clip]{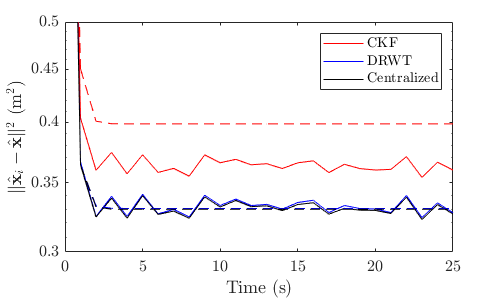}
    \caption{Mean squared error of estimation methods on a 100 node, 400 edge network with respect to ground truth, averaged over 4000 Monte Carlo simulations. Solid lines show the indicate mean squared error, while dashed lines represent estimated covariances, computed as $\text{trace}(\hat{\*P})$.}
    \label{fig:mc-abs-error}
\end{figure}

\subsection{CARLA Simulations}
We demonstrate our algorithm in a scenario involving a network of 50 sensor vehicles and 50 target vehicles within CARLA \cite{Dosovitskiy17}, a simulation test-bed for autonomous driving systems. For the simulation trials, each sensor vehicle is equipped with a forward and a backward-facing camera, each with a $90^\circ$ field of view. As shown in Figure \ref{fig:carla-frame}, sensor vehicles acquire semantic segmentation and depth images at $4 \: \text{Hz}$. The sensing radius of the vehicles is limited to 100m.

\begin{figure}[thpb]
    \centering
    \includegraphics[width=0.9\columnwidth]{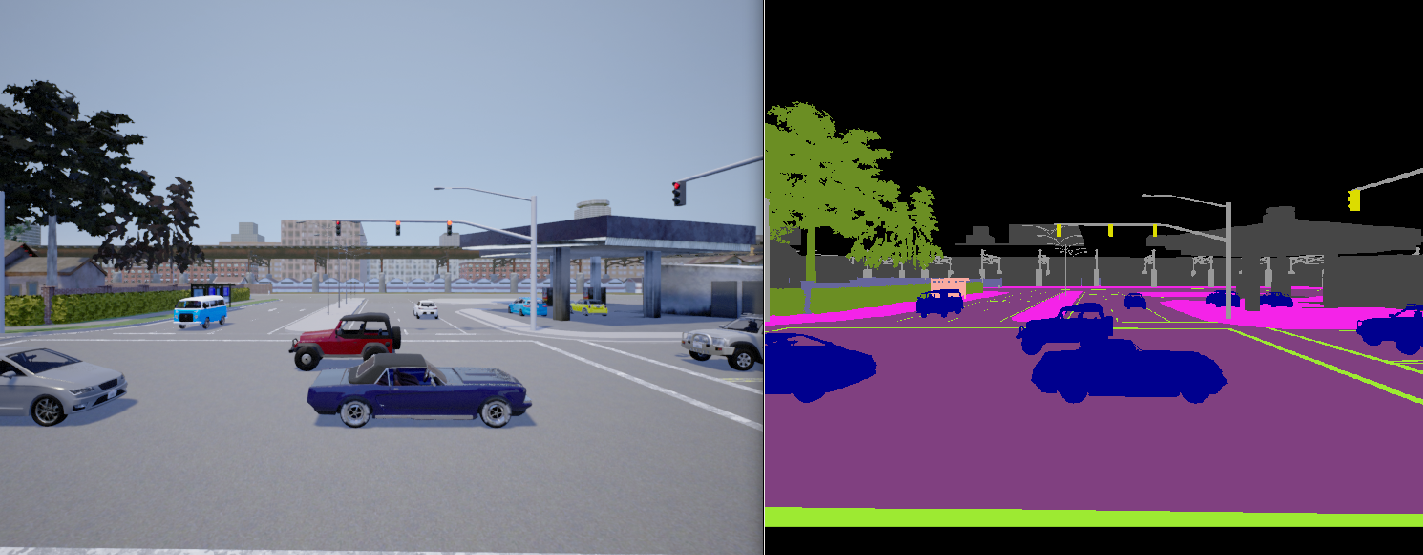} 
    \caption{CARLA frame showing raw and segmented camera images.}
    \label{fig:carla-frame}
\end{figure}

The relative position of each target vehicle is deduced from the depth and segmentation images and the camera's projection matrix. Each sensor uses its odometry information to transform the relative position of the target into the global coordinate frame corresponding to the measurement used by the vehicle in DRWT. The sensor estimates trajectories of $T = 5\text{s}$ in length. For this simulation, we assume that the target labeling is known \textit{a priori}. The communication network between sensor vehicles is modeled as a disk graph with a $200\text{m}$ radius and is updated at $4 \: \text{Hz}$. DRWT uses a simple double integrator model for the vehicle dynamics.

\begin{figure}[thpb]
    \centering
    \includegraphics[width=\columnwidth]{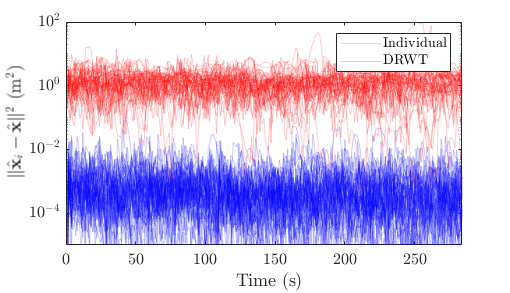} 
    \caption{Mean squared error to the centralized estimate across the full trajectories of all 50 targets. Red lines are the positional estimate errors for each individual sensor (with no communication), and the blue lines are for the DWRT positional estimates.}
    \label{fig:log-error-carla}
\end{figure}

\begin{figure}[thpb]
    \centering
    \includegraphics[width=\columnwidth]{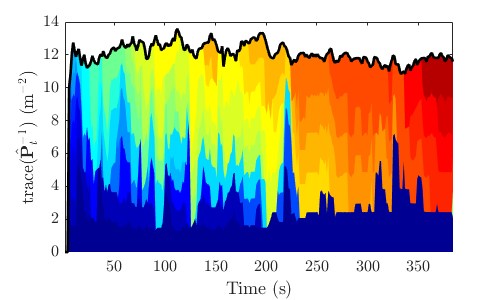}
    \caption{The sum of the traces of information matrices maintained by sensor vehicles using DRWT for a single target in a CARLA simulation. Each colored band represents the information of one sensor. Although any one sensor possesses only a fraction of the joint information, the sum over the network closely matches the information of a centralized estimator. Spikes in individual bands correspond to execution of the hand-off procedure.}
    \label{fig:info-tr-carla}
\end{figure}

Figure \ref{fig:log-error-carla} shows the mean squared error of the estimated target trajectories of all target vehicles for all the sensor vehicles with respect to the centralized trajectory estimate. Collaborative target tracking using DRWT significantly outperforms the estimates made by any single agent. Increasing the number of iterations of DRWT in each estimation round can further reduce the remaining error.

Figure \ref{fig:info-tr-carla} shows how the information (represented as the trace of the inverse covariance) corresponding to a given target is apportioned across the network. As the set of sensors tracking a target changes in time, the hand-off procedure enables their joint information to closely match the information of the centralized estimate.

\section{Conclusion}
\label{sec:conclusions}
The DRWT algorithm enables a fleet of autonomous vehicles to track other vehicles in urban environment in the presence of occlusions. In this method, each sensor-equipped vehicle estimates the target's state over a rolling window, leading to a scalable algorithm that can be parallelized to multiple targets. We show that DRWT converges to the centralized estimate even with less communication bits per node. Future work will focus on target tracking by vehicles with non-linear dynamics and non-linear sensors such as radar and lidar.

\bibliographystyle{IEEEtran}
\bibliography{Biblio}

\begin{thebibliography}{10}
\providecommand{\url}[1]{#1}
\csname url@samestyle\endcsname
\providecommand{\newblock}{\relax}
\providecommand{\bibinfo}[2]{#2}
\providecommand{\BIBentrySTDinterwordspacing}{\spaceskip=0pt\relax}
\providecommand{\BIBentryALTinterwordstretchfactor}{4}
\providecommand{\BIBentryALTinterwordspacing}{\spaceskip=\fontdimen2\font plus
\BIBentryALTinterwordstretchfactor\fontdimen3\font minus
  \fontdimen4\font\relax}
\providecommand{\BIBforeignlanguage}[2]{{%
\expandafter\ifx\csname l@#1\endcsname\relax
\typeout{** WARNING: IEEEtran.bst: No hyphenation pattern has been}%
\typeout{** loaded for the language `#1'. Using the pattern for}%
\typeout{** the default language instead.}%
\else
\language=\csname l@#1\endcsname
\fi
#2}}
\providecommand{\BIBdecl}{\relax}
\BIBdecl

\bibitem{olfati2005distributed}
R.~Olfati-Saber, ``Distributed {K}alman filter with embedded consensus
  filters,'' in \emph{Proceedings of the 44th IEEE Conference on Decision and
  Control}.\hskip 1em plus 0.5em minus 0.4em\relax IEEE, 2005, pp. 8179--8184.

\bibitem{olfati2007distributed}
------,
  ``\href{https://ieeexplore.ieee.org/stamp/stamp.jsp?tp=&arnumber=4434303}{Distributed
  {K}alman filtering for sensor networks},'' in \emph{Decision and Control,
  2007 46th IEEE Conference on}.\hskip 1em plus 0.5em minus 0.4em\relax IEEE,
  2007, pp. 5492--5498.

\bibitem{olfati2009kalman}
------, ``{K}alman-consensus filter: Optimality, stability, and performance,''
  in \emph{Proceedings of the 48h IEEE Conference on Decision and Control (CDC)
  held jointly with 2009 28th Chinese Control Conference}.\hskip 1em plus 0.5em
  minus 0.4em\relax IEEE, 2009, pp. 7036--7042.

\bibitem{battistelli2016stability}
G.~Battistelli and L.~Chisci, ``Stability of consensus extended {K}alman filter
  for distributed state estimation,'' \emph{Automatica}, vol.~68, pp. 169--178,
  2016.

\bibitem{wu2018distributed}
Z.~Wu, M.~Fu, Y.~Xu, and R.~Lu, ``A distributed {K}alman filtering algorithm
  with fast finite-time convergence for sensor networks,'' \emph{Automatica},
  vol.~95, pp. 63--72, 2018.

\bibitem{ong2006decentralised}
L.-L. Ong, B.~Upcroft, T.~Bailey, M.~Ridley, S.~Sukkarieh, and
  H.~Durrant-Whyte, ``A decentralised particle filtering algorithm for
  multi-target tracking across multiple flight vehicles,'' in \emph{2006
  IEEE/RSJ International Conference on Intelligent Robots and Systems}.\hskip
  1em plus 0.5em minus 0.4em\relax IEEE, 2006, pp. 4539--4544.

\bibitem{ahmad2011multi}
A.~Ahmad and P.~U. Lima, ``Multi-robot cooperative object tracking based on
  particle filters.'' in \emph{ECMR}.\hskip 1em plus 0.5em minus 0.4em\relax
  Citeseer, 2011, pp. 37--42.

\bibitem{carli2008distributed}
R.~Carli, A.~Chiuso, L.~Schenato, and S.~Zampieri, ``Distributed {K}alman
  filtering based on consensus strategies,'' \emph{IEEE Journal on Selected
  Areas in Communications}, vol.~26, no.~4, pp. 622--633, 2008.

\bibitem{stroupe2001distributed}
A.~W. Stroupe, M.~C. Martin, and T.~Balch, ``Distributed sensor fusion for
  object position estimation by multi-robot systems,'' in \emph{Proceedings
  2001 ICRA. IEEE International Conference on Robotics and Automation (Cat. No.
  01CH37164)}, vol.~2.\hskip 1em plus 0.5em minus 0.4em\relax IEEE, 2001, pp.
  1092--1098.

\bibitem{niehsen2002information}
W.~Niehsen, ``Information fusion based on fast covariance intersection
  filtering,'' in \emph{Proceedings of the Fifth International Conference on
  Information Fusion. FUSION 2002.(IEEE Cat. No. 02EX5997)}, vol.~2.\hskip 1em
  plus 0.5em minus 0.4em\relax IEEE, 2002, pp. 901--904.

\bibitem{julier2007using}
S.~J. Julier and J.~K. Uhlmann, ``Using covariance intersection for {SLAM},''
  \emph{Robotics and Autonomous Systems}, vol.~55, no.~1, pp. 3--20, 2007.

\bibitem{li2013cooperative}
H.~Li and F.~Nashashibi, ``Cooperative multi-vehicle localization using split
  covariance intersection filter,'' \emph{IEEE Intelligent Transportation
  Systems Magazine}, vol.~5, no.~2, pp. 33--44, 2013.

\bibitem{noack2017decentralized}
B.~Noack, J.~Sijs, M.~Reinhardt, and U.~D. Hanebeck, ``Decentralized data
  fusion with inverse covariance intersection,'' \emph{Automatica}, vol.~79,
  pp. 35--41, 2017.

\bibitem{ahmad2013cooperative}
A.~Ahmad, G.~D. Tipaldi, P.~Lima, and W.~Burgard, ``Cooperative robot
  localization and target tracking based on least squares minimization,'' in
  \emph{2013 IEEE International Conference on Robotics and Automation}.\hskip
  1em plus 0.5em minus 0.4em\relax IEEE, 2013, pp. 5696--5701.

\bibitem{nerurkar2009distributed}
E.~D. Nerurkar, S.~I. Roumeliotis, and A.~Martinelli, ``Distributed maximum a
  posteriori estimation for multi-robot cooperative localization,'' in
  \emph{2009 IEEE International Conference on Robotics and Automation}.\hskip
  1em plus 0.5em minus 0.4em\relax IEEE, 2009, pp. 1402--1409.

\bibitem{dames2017detecting}
P.~Dames, P.~Tokekar, and V.~Kumar, ``Detecting, localizing, and tracking an
  unknown number of moving targets using a team of mobile robots,'' \emph{The
  International Journal of Robotics Research}, vol.~36, no. 13-14, pp.
  1540--1553, 2017.

\bibitem{dames2017distributed}
P.~Dames, ``Distributed multi-target search and tracking using the {PHD}
  filter,'' in \emph{2017 International Symposium on Multi-Robot and
  Multi-Agent Systems (MRS)}, 2017, pp. 1--8.

\bibitem{sibley2006sliding}
G.~Sibley, ``Sliding window filters for {SLAM},'' \emph{University of Southern
  California, Tech. Rep.}, 2006.

\bibitem{boyd}
S.~Boyd, N.~Parikh, E.~Chu, B.~Peleato, J.~Eckstein \emph{et~al.},
  ``Distributed optimization and statistical learning via the alternating
  direction method of multipliers,'' \emph{Foundations and
  Trends{\textregistered} in Machine Learning}, vol.~3, no.~1, pp. 1--122,
  2011.

\bibitem{mateos}
G.~Mateos, J.~A. Bazerque, and G.~B. Giannakis, ``Distributed sparse linear
  regression,'' \emph{IEEE Transactions on Signal Processing}, vol.~58, no.~10,
  pp. 5262--5276, 2010.

\bibitem{rockafellar}
R.~T. Rockafellar, ``Monotone operators and the proximal point algorithm,''
  \emph{SIAM Journal on Control and Optimization}, vol.~14, no.~5, pp.
  877--898, 1976.

\bibitem{montijano2013distributed}
E.~Montijano, R.~Aragues, and C.~Sag{\"u}{\'e}s, ``Distributed data association
  in robotic networks with cameras and limited communications,'' \emph{IEEE
  Transactions on Robotics}, vol.~29, no.~6, pp. 1408--1423, 2013.

\bibitem{manzoor2019real}
M.~A. Manzoor, Y.~Morgan, and A.~Bais, ``Real-time vehicle make and model
  recognition system,'' \emph{Machine Learning and Knowledge Extraction},
  vol.~1, no.~2, pp. 611--629, 2019.

\bibitem{hsieh2014symmetrical}
J.-W. Hsieh, L.-C. Chen, and D.-Y. Chen, ``Symmetrical {SURF} and its
  applications to vehicle detection and vehicle make and model recognition,''
  \emph{IEEE Transactions on Intelligent Transportation Systems}, vol.~15,
  no.~1, pp. 6--20, 2014.

\bibitem{du2012automatic}
S.~Du, M.~Ibrahim, M.~Shehata, and W.~Badawy, ``Automatic license plate
  recognition: A state-of-the-art review,'' \emph{IEEE Transactions on Circuits
  and Systems for Video Technology}, vol.~23, no.~2, pp. 311--325, 2012.

\bibitem{chang2014multi}
T.-H. Chang, M.~Hong, and X.~Wang, ``Multi-agent distributed optimization via
  inexact consensus {ADMM},'' \emph{IEEE Transactions on Signal Processing},
  vol.~63, no.~2, pp. 482--497, 2014.

\bibitem{Dosovitskiy17}
A.~Dosovitskiy, G.~Ros, F.~Codevilla, A.~Lopez, and V.~Koltun, ``{CARLA}: {An}
  open urban driving simulator,'' in \emph{Proceedings of the 1st Annual
  Conference on Robot Learning}, 2017, pp. 1--16.

\end{thebibliography}

\end{document}